\documentclass{article}

\usepackage[preprint,nonatbib]{neurips_2021}

\usepackage[utf8]{inputenc} 
\usepackage[T1]{fontenc}    
\usepackage{hyperref}       
\usepackage{url}            
\usepackage{booktabs}       
\usepackage{amsfonts}       
\usepackage{nicefrac}       
\usepackage{microtype}      
\usepackage{xcolor}         

\usepackage{amsmath}
\usepackage{amssymb}
\usepackage{amsthm}
\usepackage{bm}
\usepackage{multirow}
\usepackage{graphicx}
\RequirePackage{algorithm}
\RequirePackage{algorithmic}
\usepackage[caption = false]{subfig}
\usepackage{enumitem }

\theoremstyle{theorem}
\newtheorem{theorem}{Theorem}[section]

\newtheorem{lemma}[theorem]{Lemma}

\newtheorem{proposition}[theorem]{Proposition}

\newcommand{\bd}[1]{\mathbf{#1}}

\DeclareMathOperator*{\argmin}{arg\,min}

\floatname{algorithm}{Algorithm}
\renewcommand{\algorithmicrequire}{\textbf{Input:}}
\renewcommand{\algorithmicensure}{\textbf{Output:}}

\title{One Representation to Rule Them All: Identifying Out-of-Support Examples in Few-shot Learning with Generic Representations}

\author{%
  Henry Kvinge, Scott Howland, Nico Courts,\\ {\textbf{Lauren A. Phillips, John Buckheit, Zachary New}}\\ {\textbf{Elliott Skomski, Jung H. Lee, Aaron Tuor,}}\\ {\textbf{Sandeep Tiwari, Jessica Hibler, Courtney D. Corley,}}\\ {\textbf{Nathan O. Hodas}} \\
  Pacific Northwest National Laboratory\\
  Seattle, WA, USA\\
  \texttt{first.last@pnnl.gov} \\
}

\begin{document}

\maketitle

\begin{abstract}
The field of few-shot learning has made remarkable strides in developing powerful models that can operate in the small data regime. Nearly all of these methods assume every unlabeled instance encountered will belong to a handful of known classes for which one has examples. This can be problematic for real-world use cases where one routinely finds `none-of-the-above' examples. In this paper we describe this challenge of identifying what we term `out-of-support' (OOS) examples. We describe how this problem is subtly different from out-of-distribution detection and describe a new method of identifying OOS examples within the Prototypical Networks framework using a fixed point which we call the generic representation. We show that our method outperforms other existing approaches in the literature as well as other approaches that we propose in this paper. Finally, we investigate how the use of such a generic point affects the geometry of a model's feature space.
\end{abstract}

\section{Introduction}
\label{sect-introduction}

Over the past decade, deep learning-based methods have achieved state-of-the-art performance in a range of applications including image recognition, speech recognition, and machine translation. There are many problems however, where deep learning's utility remains limited because of its need for large amounts of labeled data \cite{LeCun2015}. The field of few-shot learning \cite{Wang2018} aims to develop methods for building powerful machine learning models in the limited-data regime.

The common paradigm in few-shot learning is to assume that for each unlabeled instance, one has at least one labeled example belonging to the same class. At inference time then, classification of an unlabeled example $x$ simply involves determining which of a fixed number of known classes $x$ is most likely to belong to. In real-world problems on the other hand, it is frequently the case that one does not have labeled examples of every possible class that has support in a data distribution. This is particularly true in science and medical applications where it is time and cost prohibitive to have a subject matter expert sift through an entire dataset and identify all classes therein. Establishing methods of detecting whether or not unlabeled input belongs to any known class is thus critical to making few-shot learning an effective tool in a broad range of applications.

We define a datapoint to be {\emph{out-of-support (OOS)}} if it does not belong to a class for which we have labeled examples, but was still drawn from the same data distribution as the labeled examples we have. We call the problem of identifying such instances the {\emph{out-of-support detection problem}}. As we explain in Section \ref{subsect-OOD-dectection}, OOS detection resembles, but is distinct from, out-of-distribution (OOD) detection where one attempts to identify examples which were drawn from a different data distribution entirely, (see Figure \ref{fig-ood-vs-oos} for an illustration of the difference between these two types of problems). To our knowledge the OOS detection problem was first articulated in the literature only recently in \cite{wang2019out}, where two algorithms were proposed within the metric-based few-shot setting.

In this paper we describe a new approach to OOS detection which we call {\emph{Generic Representation Out-Of-Support (GROOS) Detection}}. The name is inspired by the concept of generic points in algebraic geometry, which are points for which all generic properties of a geometric object are true \cite{hartshorne2013algebraic}. Our method uses a so-called {\emph{generic representation}} to represent the data distribution as a whole but no individual class in particular. Like the methods proposed in \cite{wang2019out}, our method can be adapted to work with a range of metric-based few-shot models. For simplicity, in this paper we focus on a Prototypical Networks \cite{snell2017prototypical} setting where the generic representation is simply a point in feature space. To predict whether an unlabeled instance $q$ is OOS or not, one compares the distances from the encoding of $q$ to each class representation and the generic representation. If the image of $q$ is sufficiently close to the generic representation and sufficiently far from all class representations, it is predicted to be OOS. We state a pair of inequalities \eqref{eqn-generic-condition} relating the distances between query points, class prototypes, and the generic representation which need to be satisfied in order for GROOS detection to be able to correctly predict when $q$ is OOS and also correctly predict the class of $q$ when $q$ is in-support. We analyze how these constraints effect the geometry of a model's feature space, characterizing its structure through three Propositions (Propositions \ref{prop-decomp}, \ref{prop-viable-regions}, and \ref{prop-relationship between cells}). We also show that for GROOS to be successful, additional `second-order' relationships between prototypes and the generic representation need to hold.

We benchmark GROOS detection against two recently proposed methods - LCBO and MinDist \cite{wang2019out} - as well as an additional method - Background OOS detection - which we describe in this paper. We find that GROOS detection not only on average outperforms previous benchmarks (Section \ref{sect-experiments-standard}), but an adapted version of GROOS called {\emph{Centered GROOS}} tends to outperform other OOS detection methods in settings that require significant model generalization (Section \ref{subsect-generalization-experiments}). Despite the strong relative performance of Centered GROOS detection in this latter setting, it is clear that the community still has a considerable amount of work to do before few-shot models can satisfactorily detect OOS examples when evaluated on datasets significantly different from those that they were trained on.

In summary, our contributions in this paper include:
\begin{itemize}[leftmargin=13pt,noitemsep,topsep=0pt]
\item We introduce the GROOS detection method, which is designed to solve the out-of-support detection problem in few-shot learning using a generic representation.
\item We benchmark GROOS detection against existing metric-based methods in the literature and an additional OOS detection method, Background OOS Detection, which we describe in this paper. We show that GROOS out-performs these approaches both in a traditional few-shot train-evaluation setting, and in a more challenging setting where models are trained on ImageNet and then evaluated on a diverse range of datasets.
\item We state two inequalities relating class prototypes, the generic representation, and encoded query points, which must be satisfied in order for both OOS detection and standard in-support classification to be effective. Motivated by these inequalities we prove three propositions which relate feature space geometries that arise from the standard Prototypical Networks problem and the feature space geometries that arise from GROOS.
\end{itemize}

\section{Background and related work}
\label{sect-related-work}


\subsection{Few-shot learning and Prototypical Networks}
\label{sect-protonets}

There are a range of approaches that have been used to address the challenges of few-shot learning. Fine-tuning methods \cite{chen2019closer,chowdhury2021few} use transfer learning followed by fine-tuning to train models with limited data. Data augmentation methods \cite{hariharan2017low} leverage augmentation and generative approaches to produce additional training data. Gradient-based meta-learning \cite{finn2017model,nichol2018first} is a class of methods that use sophisticated optimization techniques to learn strong initial weights which can be adapted to a new task with a small number of gradient steps. The algorithm we propose in this paper is related to a fourth class of algorithms called metric-based models. In these models an encoder function is trained to embed data into a space where a distance metric (either hard-coded or learned) captures some task-appropriate notion of similarity. Well-known examples of metric-based few-shot models include Prototypical Networks \cite{snell2017prototypical}, Matching Networks \cite{NIPS2016_6385}, and Relational Network \cite{sung2018learning}.



An {\emph{episode}} is the basic unit of few-shot inference and training. It consists of a set $S$ of labelled examples known as the {\emph{support set}} and an unlabeled set $Q$ known as the {\emph{query set}}. Within an episode, a model uses elements in $S$ to predict labels for elements in $Q$. We assume that elements of $S$ belong to classes $C^{in} = \{1,\dots,k\}$. For convenience, we decompose $S$ into a disjoint union: $S = \bigcup_{c \in C^{in}} S_c$, where $S_c$ contains only those elements of $S$ with label $c$. We will assume that the size $n = |S_c|$ is constant for all $c \in C^{in}$. The integer $n$ is known as the {\emph{shots}} of the episode, while the integer $k$ is known as the {\emph{ways}}. In this paper we will assume that $Q$ has been drawn from a distribution $p$, and each set $S_c$ has been drawn from the conditional distribution $p(y = c)$. 

By {\emph{few-shot training}} we mean the process of calculating the loss for an entire episode and then using that loss to update the weights of the model. {\emph{Few-shot inference}} has an analogous meaning. A {\emph{few-shot split}} is a partition of a dataset into train and test sets by class, so that examples from each class are contained in either the train or test split, but not in both. 

Prototypical networks (ProtoNets) \cite{snell2017prototypical} uses an encoder function $f: X \rightarrow \mathbb{R}^d$ to map elements of both $Q$ and $S$ into metric space $\mathbb{R}^d$ (which we will always assume is equipped with the Euclidean metric). In $\mathbb{R}^d$, a centroid $\gamma_c$ is formed for each set $f_\theta(S_c)$. $\gamma_c$ is referred to as the {\emph{prototype}} which represents class $c$ in $\mathbb{R}^d$. The model predicts the class of an unlabelled query point $q$ based on the solution to $\argmin_{c \in C^{in}}||\gamma_c - f_\theta(q)||$. Note that in the case where one needs probabilities associated with a prediction, one can apply a softmax function to the distance vector $[-d_c]_{c \in C^{in}}$ where $d_c = ||\gamma_c - f_\theta(x)||$. 

\subsection{The out-of-support detection problem}
\label{subsect-oos}

As mentioned in the Introduction, it is commonly assumed in the literature that all elements of $Q$ have a label from $C^{in}$. It was observed in \cite{wang2019out} that in many real-world cases, this assumption is unrealistic. In that work the authors referred to an example $q \in Q$ that does not belong to a class in $C^{in}$ as being ``out-of-episode''. We feel it is more appropriate to describe such examples as being {\emph{out-of-support}} (OOS), since any elements found in $Q$ can be said to be part of the episode. Following \cite{wang2019out} we decompose $Q$ as $Q = Q^{in} \cup Q^{out}$ where $Q^{in}$ are those elements that are in-support and $Q^{out}$ are those elements that are OOS. It is also convenient to use $C$ to denote the set of all labels on elements from $S\cup Q$, with $C$ decomposing as the disjoint union $C = C^{in} \cup C^{out}$ where $C^{out}$ are simply those classes for which there are unlabeled examples in $Q$ but no labeled examples in $S$. Note that the user generally does not have knowledge of $C^{out}$.

The {\emph{out-of-support (OOS) detection problem}} then involves identifying those elements of $Q$ that do not belong to any class in $C^{in}$. All the methods for OOS detection described or introduced in this paper use a confidence score $\varphi: X \rightarrow \mathbb{R}$ that maps a query point $q \in Q$ to a value in $\mathbb{R}$. In general, $\varphi$ also depends on the full support set $S$ as well as the encoder $f_\theta$, but to simplify notation we assume these dependencies are implicit. 

The authors of \cite{wang2019out} proposed two methods for OOS detection. Both are presented as an additional component that can be added to Prototypical Networks and it is in this context that we will describe and evaluate them. The first uses a function called the {\emph{Minimum Distance Confidence Score}} (MinDist), $\varphi_{dist}: X \rightarrow \mathbb{R}$ which is defined as $\varphi_{dist}(q,f_\theta) = -\min_{c \in C^{in}}||\gamma_c - f_\theta(q)||$. A query $q$ is predicted to be OOS if $\varphi_{dist}(q) < t$ for some $t < 0$. The second method proposed in \cite{wang2019out} is the {\emph{Learnable Class BOundary (LCBO) Network}} which is a parametric class-conditional confidence score $\varphi_{LC}: X \rightarrow \mathbb{R}$. $\varphi_{LC}$ uses a small, fully-connected neural network $h_{\theta'}: \mathbb{R}^d \rightarrow \mathbb{R}$ to produce scores for each prototype/query pair $(q,\gamma_c)$. The confidence score $\varphi_{LC}$ is defined as $\varphi_{LC}(q) = \max_{c \in C^{in}}\big(h_{\theta'}(\gamma_c,f_\theta(q))\big)$. The model predicts that $q$ is OOS if $\varphi_{LC}(q) < t$ for some predetermined threshold $t$. We note that the authors of \cite{wang2019out} used an additional term in the loss function to encourage their models to correctly identify OOS examples. We did not find the addition of such loss terms necessary to achieve good performance with the models introduced in this paper. 

\subsection{Out-of-distribution detection}
\label{subsect-OOD-dectection}

Out-of-distribution (OOD) detection aims to develop methods that can identify whether or not a data point $x$ was drawn from some known distribution $p$. Methods for doing this within the context of deep learning models include: using a model's largest softmax output value as a confidence score \cite{hendrycks2016baseline, lakshminarayanan2017simple} and ODIN \cite{lee2018training} which suggests identifying OOD examples through the use of model gradients and sofmax temperature scaling. Standard benchmarks for OOD detection focus on using OOD detection methods to identify examples drawn from very visually distinct distributions. For example, a common experiment attempts to detect Gaussian noise or MNIST \cite{lecun1998gradient} images injected into the CIFAR10 dataset \cite{krizhevsky2009learning}.
 
OOS detection differs from OOD detection in that, in general, the conditional distributions corresponding to elements belonging to $C^{in}$ and $C^{out}$ only vary in subtle and arbitrary ways. Consider the example summarized in Figure \ref{fig-ood-vs-oos} where $S$ and $Q$ consist of images of dogs. While it is true that the distribution of dog images belonging to classes $C^{in} = \{\text{Newfoundland, pug}\}$ is different than those belonging to classes $C^{out} = \{\text{Labrador, Tibertan terrier}\}$, these differences are slight (and focus on very specific aspects of the input) relative to differences in distribution that OOD detection methods are designed to detect. Indeed, \cite{wang2019out} showed that a few-shot analogue of \cite{hendrycks2016baseline} applied to a ProtoNet model struggled on the OOS detection problem. Additionally, OOD detection methods generally assume that even if a model has not seen examples of OOD data, it has seen many examples of in-distribution data. This is not the case for few-shot models which only have a handful of classes that they can use to characterize ``in-distribution''. In fact, in the generalization-focused evaluation setting described in Section \ref{subsect-generalization-experiments}, few-shot models could be described as operating exclusively out-of-distribution in relation to their training set. Finally, while OOD examples are defined with respect to an entire dataset, OOS examples are only defined via a small support set, and this definition can vary from episode to episode.  As suggested in \cite{wang2019out}, all these differences argue for identifying few-shot OOS detection as a problem which is distinct from OOD detection, requiring its own set of methods.

\begin{figure}[t]
\centering
\includegraphics[width=.98\columnwidth]{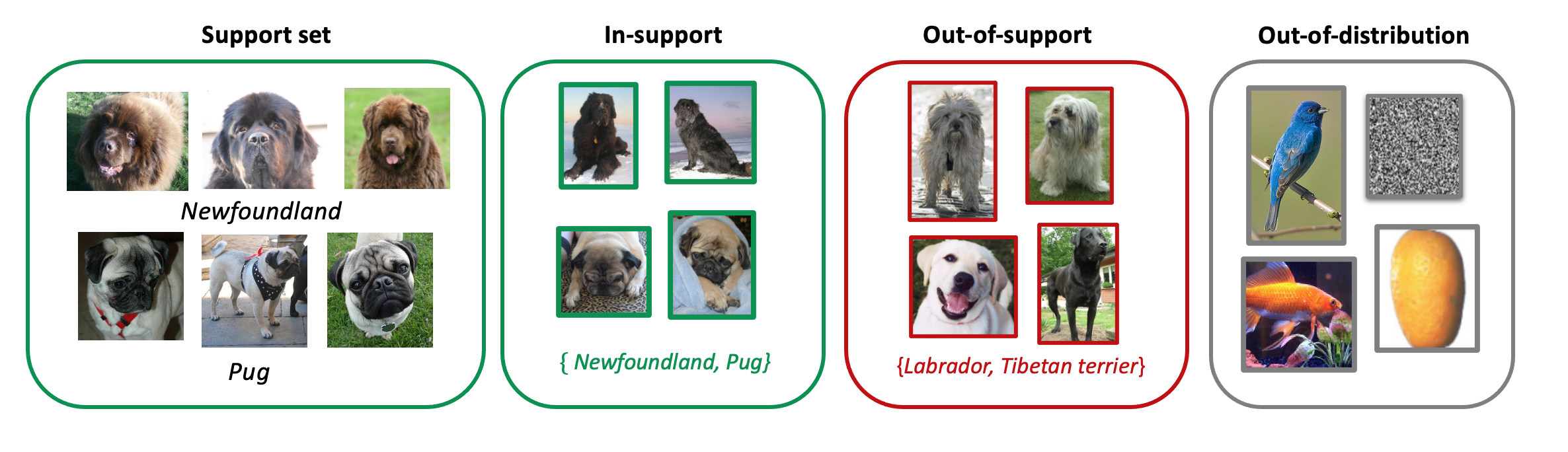} 
\caption{A diagram illustrating the difference between out-of-support detection and out-of-distribution detection for a few-shot task where one attempts to identify images of Newfoundlands and pugs from a dataset of dog images.}
\label{fig-ood-vs-oos}
\end{figure}

\section{OOS detection with a generic point}
\label{sect-algo-descript}

In this section we describe our proposed {\emph{Generic Representation Out-Of-Support (GROOS) Detection}} method. Let $f_\theta: X \rightarrow \mathbb{R}^d$ be the encoder (for example, when $X$ is an image space, then $f_\theta$ might be a ResNet \cite{resnet2016} with the final linear classification layer removed). Let $L: \mathbb{R}^d \rightarrow \mathbb{R}^d$ be an affine map, so that $L(x) = Wx + b$ for some matrix (weights) $W$ and vector (bias) $b$. We construct a new encoder by composing $h_\theta = L \circ f_\theta: X \rightarrow \mathbb{R}^d$. 

Next choose a point $\gamma_{oos} \in \mathbb{R}^d$ which will be called the {\emph{generic representation}} and a threshold $0 \leq t \leq 1$. There are many potential choices for $\gamma_{oos}$ but we find that the origin works well in practice. Inference with $h_\theta$ is similar to inference with the standard ProtoNets (Section \ref{sect-protonets}). For an $n$-shot, $k$-way support set $S = \cup_{c \in C^{in}} S_c$, with support set labels $C^{in} = \{1,\dots,k\}$, and query $q$, $h_\theta(S)$ and $h_\theta(q)$ are calculated and centroid prototypes $\gamma_c$ are computed for each set $h_\theta(S_c)$ with $c \in C^{in}$. We compute the vector $\bd{d}_q := (d_1, \dots, d_k,d_{oos})$ where $d_i := ||\gamma_i - h_\theta(q)||.$ Finally, let $softmax: \mathbb{R}^{k+1} \rightarrow \mathbb{R}^{k+1}$ be the standard softmax function. Following the notation in Section \ref{subsect-oos} we define $\varphi_{gen}: X \rightarrow \mathbb{R}$ to be $\varphi_{gen}(q) :=[softmax(-\bd{d}_q)]_{k+1}$ where $[softmax(-\bd{d}_q)]_{k+1}$ is the $(k+1)$st output coordinate corresponding to encoded query distance from $\gamma_{oos}$. If $\varphi_{gen}(q) > t$ then we predict that $q$ is OOS. If $\varphi_{gen}(q) < t$, then we predict that $q$ is in-support and we use the other $k$ softmax outputs from $softmax(-\bd{d}_q)$ to predict its class. Informally, this process consists of comparing the distance of the encoded query point from the generic representation to its distance to other support prototypes. If the query is sufficiently closer to the generic representation than it is to other prototypes, then it is predicted as OOS. This process is summarized in Algorithm \ref{algo-oos-prototype}.

\begin{algorithm} 
\caption{Generic Representation Out-Of-Support (GROOS) Detection} 
\label{algo-oos-prototype} 
\begin{algorithmic}
\REQUIRE{Encoder function $h_\theta: X \rightarrow \mathbb{R}^d$, generic representation $\gamma_{oos} \in \mathbb{R}^d$, support set $S = S_1 \cup \dots \cup S_k$ with corresponding label set $C^{in} = \{1,\dots,k\}$, query $q$, threshold $0 \leq t \leq 1$.}
    \FOR{$c \in C^{in}$}
        \STATE Compute prototype centroid $\gamma_c$ from $h_\theta(S_c)$
        \STATE Compute $d_{c} = ||\gamma_c - h_\theta(q)||$
    \ENDFOR
    \STATE Compute $d_{oos} = ||\gamma_{oss} - h_\theta(q)||$
        \STATE Set $\bd{d}_q = (d_{1}, \dots, d_{k},d_{oos})$ and compute $\varphi_{gen}(q) = [Softmax(-\bd{d}_q)]_{k+1}$
        \IF{$\varphi_{gen}(q) > t$}
            \STATE $q$ is predicted as OOS
        \ELSE 
            \STATE $q$ is predicted as in-support, belonging to class $c^* = \argmin_{c \in C^{in}}d_{c}$.
    	\ENDIF
\end{algorithmic}
\end{algorithm}


One can ask what metric properties an encoded dataset $h_\theta(D)$ must have in order for GROOS detection to be effective. For simplicity we assume that prototypes $\gamma_1,\dots,\gamma_k$ and generic representation $\gamma_{oos}$ are fixed (empirically we find that prototypes are fairly stable when the number of shots is high enough so this is not an unreasonable approximation). (1) To ensure in-support examples are always predicted correctly, $h_\theta$ must map any $x \in D$ with label $c \in C$ closer to $\gamma_c$ than to $\gamma_1, \dots, \gamma_{c-1}, \gamma_{c+1}, \dots, \gamma_{k}, \gamma_{oos}$. That is $||h_\theta(x) - \gamma_c|| < ||h_\theta(x) - \gamma_{c'}||$ for all $c' \in C\cup \{oos\}$ such that $c \neq c'$. (2) One the other hand, when $c$ is not represented in the support set, then $h_\theta(x)$ must be closer to $\gamma_{oos}$ than to any other class prototype which is not $\gamma_c$ (which does not appear in the episode). Specifically, $||h_\theta(x) - \gamma_{oos}|| < ||h_\theta(x) - \gamma_{c'}||$ for all $c' \in C$ such that $c' \neq c$. These inequalities can be combined for the single expression
\begin{equation}\label{eqn-generic-condition}
    ||h_\theta(x) - \gamma_c|| < ||h_\theta(x) - \gamma_{oos}||< ||h_\theta(x) - \gamma_{c'}|| \quad \text{for } c' \in C, c' \neq c.
\end{equation}

Inequality \eqref{eqn-generic-condition} suggests that if one is not able to actually train $h_\theta$ on dataset $D$ (or a similar dataset), and hence $h_
\theta$ is not able to learn how to arrange encoded data around $\gamma_{oos}$, then another sensible option is to choose $\gamma_{oos}$ to be the centroid of $h_\theta(S\cup Q)$. We call this alternative version of GROOS detection {\emph{Centered GROOS Detection}}. We will see that it works better than the standard version of GROOS detection when the test set differs significantly from the training set.


\subsection{Background detection}

We introduce a second OOS detection model to serve as an additional benchmark. We call it {\emph{Background Detection}} since it was inspired by the ``background class'' described in \cite{zhang2017universum}. Background detection consists of an encoder function $h_\theta: X \rightarrow \mathbb{R}^d$ such as a ResNet, with its final classification layer replaced with a linear layer $L: \mathbb{R}^d \rightarrow \mathbb{R}^d$ and two predetermined constants $M > 0$ and $0 \leq t \leq 1$. An episode with support set $S = \cup_{c \in C^{in}}S_k$ and query $q$ proceeds with the usual calculation of class centroids $\gamma_c$ for $c \in C^{in}$. Using constant $M$ and distances $||\gamma_c - h_\theta(q)||$ between encoded query and prototypes the vector $\bd{d}_q := (d_1, \dots, d_k,M)$ is obtained. The confidence function $\varphi_{back}: X \rightarrow \mathbb{R}$ associated with this method is then: $\varphi_{back}(q) := \big[softmax(-\bd{d}_{q})\big]_{k+1}$. Query $q$ is predicted to be OOS if $\varphi_{back}(q) > t$.

\section{Experiments and analysis}
\label{sect-experiments}

\subsection{Standard few-shot evaluation}
\label{sect-experiments-standard}

Our first set of experiments look at how well OOS detection methods (MinDist, LCBO, Background Detection, GROOS, and Centered GROOS) can identify OOS examples in the setting where the base model is trained and evaluated on few-shot splits drawn from the same dataset. That is, we partition the classes of the dataset between train and test. We focus on the datasets: CIFAR100 \cite{krizhevsky2009learning} (CC-BY 4.0), CUB-200 \cite{WahCUB_200_2011} (CC0 1.0), and Omniglot \cite{Lake1332} (MIT License).

All models were trained for four days of wall clock time on a single Tesla P100 GPU for a total of between 250,000 and 500,000 training episodes in that time. All performances stabilized around the lower end of that range. All models used a ResNet18 encoder with the final linear layer removed and were initialized with the standard ImageNet (CC-BY 4.0) pre-trained weights from Torchvision \cite{marcel2010torchvision}. We address the question of how performance differs for different sizes of encoder in Section \ref{appendix-encoder-size} of the Appendix. We used the Adam optimizer for training, with a learning rate of $1\times 10^5$, a weight decay factor of $5\times 10^{-5}$, and $\beta$ values of $0.9$ and $0.999$. All results correspond to $5$-shot, $5$-way episodes, with $8$ queries per support class and a total of $40$ OOS images introduced per episode (that is, $50\%$ of all images in the query were OOS). All images were resized to $224\times224$ before being fed through the model. To evaluate each model, we sampled $1000$ episodes from the corresponding few-shot test set. To complete the evaluation, we computed the area under precision recall curve (AUPR) and area under the ROC curve (AUROC) for each model with respect to the evaluation queries and multiplied these by $100$.

The result of these experiments can be found in Table \ref{fig-same-dataset}. We bold all scores that are within $0.5$ of the top model (in terms of both AUPR and AUROC), putting an $*$ on the top score for each column. As can be seen, in two of the three datasets used, GROOS outperforms other methods by at least $1.0$ both in terms of AUPR and AUROC. On Omniglot, MinDist, LCBO, GROOS, and Centered GROOS all do close to perfect. We include this last experiment to demonstrate that when a sufficiently strong encoder is used for a simpler dataset, then a range of OOS detection methods can do quite well.   

\begin{table*}
    \centering
    {\def\arraystretch{1.5}%
    \setlength{\tabcolsep}{4pt}
    \begin{tabular}{ccccccc}
 & \multicolumn{2}{c}{CIFAR100}  & \multicolumn{2}{c}{CUB-200} & \multicolumn{2}{c}{Omniglot} \\
 & AUPR & AUROC  & AUPR & AUROC  & AUPR & AUROC \\
 \hline
MinDist & {\small{88.4$\pm$0.1}} & {\small{88.6$\pm$0.1}} & {\small{89.0$\pm$0.2}} & {\small{89.2$\pm$0.2}} & {\small{\textbf{99.4$\pm$0.1}}} & {\small{\textbf{99.5$\pm$0.1$^*$}}}  \\
LCBO & {\small{83.2$\pm$0.5}} & {\small{84.7$\pm$0.4}} & {\small{85.8$\pm$0.5}}  & {\small{87.6$\pm$0.3}} & {\small{\textbf{99.2$\pm$0.2}}} & {\small{\textbf{99.3$\pm$0.1}}}    \\
\hline
\textbf{Background (ours)} & {\small{87.2$\pm$0.3}} & {\small{86.9$\pm$0.3}} & {\small{88.8$\pm$0.7}} & {\small{88.0$\pm$0.7}} & {\small{98.9$\pm$0.1}} & {\small{98.9$\pm$0.1}}\\
\textbf{GROOS (ours)} & {\small{\textbf{90.1$\pm$0.2$^*$}}} & {\small{\textbf{90.2$\pm$0.3$^*$}}} & {\small{\textbf{90.9$\pm$0.6$^*$}}} & {\small{\textbf{90.6$\pm$0.7$^*$}}} & {\small{\textbf{99.6$\pm$0.1$^*$}}} & {\small{\textbf{99.5$\pm$0.1$^*$}}}  \\
\textbf{Centered GROOS (ours)} & {\small{88.9$\pm$0.9}} & {\small{88.4$\pm$0.8}} & {\small{89.6$\pm$0.2}}  & {\small{89.5$\pm$0.1}} & {\small{\textbf{99.5$\pm$0.1}}} & {\small{\textbf{99.5$\pm$0.1$^*$}}} \\
\hline
\end{tabular}}
    \caption{The area under the ROC curve (AUROC) and area under the precision-recall curve (AUPR) for a range of few-shot OOS detection methods. We put an $*$ next to the top score in each column and set in bold all the rest of the scores in the column that are within $0.5$ of this.}
    \label{fig-same-dataset}
\end{table*}


\subsection{Generalization experiments}
\label{subsect-generalization-experiments}

The experiments in the previous section simulated the situation where one has access to a training dataset which is similar to the data that one wants to apply the model to during inference. However, as pointed out in the Introduction, there are many applications of few-shot learning where one does not have access to such a training set. In these cases it is important that a model can perform well, even when the dataset one wants to perform inference on is substantially different from the data used for training. 

We ran experiments to evaluate how adaptable MinDist, LCBO, Background Detection, GROOS Detection, and Centered GROOS Detection were when a dataset from a previously unseen distribution was introduced at inference time. We chose to train our networks on a few-shot training split of ImageNet, as ImageNet has been shown to generally produce rich and flexible feature extractors \cite{kornblith2019better,chowdhury2021few}, and then test on: the few-shot ImageNet testset, CIFAR100, Omniglot, Aircraft \cite{maji13fine-grained}, Describable Textures \cite{cimpoi14describing}, and Fruits 360 \cite{murecsan2018fruit}.\footnote{Aircraft is available exclusively for non-commercial research purposes, Describable Textures is available for research purposes, and Fruits 360 is covered by 
CC BY-SA 4.0} All models used the same encoder, hyperparameters, and training scheme as that described in Section \ref{sect-experiments-standard}. 

\begin{table*}
    \centering
    {\def\arraystretch{1.5}%
    \setlength{\tabcolsep}{4pt}
    \begin{tabular}{cccccccc}
&  ImageNet & CIFAR100  & Omniglot &  Aircraft & Textures & Fruits\\
 \hline
MinDist & {\small{95.0$\pm$0.1}} & {\small{\textbf{80.1$\pm$0.6}}} &  {\small{\textbf{85.5$\pm$0.7$^*$}}} & {\small{59.4$\pm$0.3}} & {\small{72.7$\pm$0.8}} & {\small{95.3$\pm$0.4}}  \\
LCBO & {\small{92.6$\pm$0.1}} & {\small{76.3$\pm$1.0}} & {\small{68.5$\pm$1.6}} & {\small{54.7$\pm$0.5}} & {\small{65.3$\pm$1.5}} & {\small{89.0$\pm$2.2}}   \\
\hline
\textbf{Background (ours)} & {\small{93.6$\pm$0.1}}  & {\small{77.5$\pm$0.9}} & {\small{58.6$\pm$1.9}} & {\small{58.4$\pm$0.4}} & {\small{71.8$\pm$0.8}} & {\small{89.8$\pm$1.0}} \\
\textbf{GROOS (ours)} & {\small{\textbf{95.7$\pm$0.1$^*$}}} & {\small{74.3$\pm$3.0}} & {\small{74.6$\pm$1.4}} & {\small{53.8$\pm$0.3}} & {\small{75.2$\pm$1.0}} & {\small{91.2$\pm$0.8}}  \\
\textbf{Centered GROOS (ours)} & {\small{95.0$\pm$0.2}} & {\small{\textbf{80.6$\pm$0.8$^*$}}} & {\small{82.1$\pm$0.5}} &  {\small{\textbf{61.8$\pm$4.7$^*$}}} & {\small{\textbf{82.3$\pm$0.2$^*$}}} & {\small{\textbf{96.2$\pm$0.3$^*$}}} \\
\hline
\end{tabular}}
    \caption{The area under the ROC curve (AUROC) for a range of few-shot OOS detection methods which were all trained on a few-shot training split of ImageNet and then evaluated on a range of datasets. We put an $*$ next to the top score in each column and set in bold all the rest of the scores in the column that are within $0.5$ of this.}
    \label{fig-generalization-AUROC}
\end{table*}

\begin{table*}
    \centering
    {\def\arraystretch{1.5}%
    \setlength{\tabcolsep}{4pt}
    \begin{tabular}{cccccccc}
&  ImageNet & CIFAR100  & Omniglot &  Aircraft & Textures & Fruits\\
 \hline
MinDist & {\small{\textbf{95.0$\pm$0.1}}} & {\small{\textbf{79.4$\pm$0.5}}} & {\small{\textbf{86.3$\pm$0.7$^*$}}} & {\small{59.3$\pm$0.2}} & {\small{73.8$\pm$0.8}} & {\small{95.5$\pm$0.3}}  \\
LCBO & {\small{91.8$\pm$0.1}} & {\small{73.3$\pm$1.4}} & {\small{66.4$\pm$1.8}} & {\small{53.8$\pm$0.4}} & {\small{64.5$\pm$1.8}} & {\small{88.6$\pm$2.6}}  \\
\hline
\textbf{Background (ours)} & {\small{93.7$\pm$0.0}} & {\small{77.1$\pm$0.8}} & {\small{76.6$\pm$0.7}} & {\small{58.3$\pm$0.3}} & {\small{70.1$\pm$0.7}} & {\small{92.8$\pm$0.5}} \\
\textbf{GROOS (ours)} & {\small{\textbf{95.5$\pm$0.1$^*$}}} & {\small{72.1$\pm$2.6}} & {\small{75.7$\pm$1.2}} & {\small{54.3$\pm$0.3}} & {\small{71.1$\pm$0.9}} & {\small{92.1$\pm$0.6}} \\
\textbf{Centered GROOS (ours)} & {\small{94.7$\pm$0.3}} & {\small{\textbf{79.8$\pm$0.9$^*$}}} & {\small{82.1$\pm$0.7}} & {\small{\textbf{61.7$\pm$0.4$^*$}}} & {\small{\textbf{79.9$\pm$0.3$^*$}}} & {\small{\textbf{96.4$\pm$0.1$^*$}}} \\
\hline
\end{tabular}}
    \caption{The area under the precision recall curve (AUPR) for a range of few-shot OOS detection methods which were all trained on a few-shot training split of ImageNet and then evaluated on a range of test datasets. We put an $*$ next to the top score in each column and set in bold all the rest of the scores in the column that are within $0.5$ of this.}
    \label{fig-generalization-AUPR}
\end{table*}

We find that in this setting, performance is generally worse for all model types. This is not surprising since the models are essentially operating on out-of-distribution data at inference time. Aircraft is a particularly challenging dataset for models that have not seen the corresponding training set. A comparison of the error bars in Tables \ref{fig-generalization-AUROC} and \ref{fig-generalization-AUPR} on the one hand and Table \ref{fig-same-dataset} on the other illustrates that when operating on OOD data there is more variation between training runs. Nonetheless, Centered GROOS detection performs better than other methods on $4$ out of the $5$ OOD datasets, with the exception of Omniglot where MinDist does substantially better. On the in-distribution test set ImageNet test, GROOS achieves better performance than Centered GROOS, confirming our hypothesis that GROOS is better to use when inference data aligns with training data and Centered GROOS is better otherwise. Of all the datasets presented to the models in these tests, Omniglot is probably the most ``unlike'' ImageNet. We conjecture that for mildly OOD datasets such as CIFAR100, Aircraft, and Fruits, Centered GROOS tends to perform better, while for significantly OOD datasets such as Omniglot, the simpler MinDist model might be a better choice.

\subsection{Generic points: feature space geometry and decision boundaries}
\label{subsect-how-do-prototypes-work}

\begin{figure*}[t]
\centering
\includegraphics[width=.3\columnwidth]{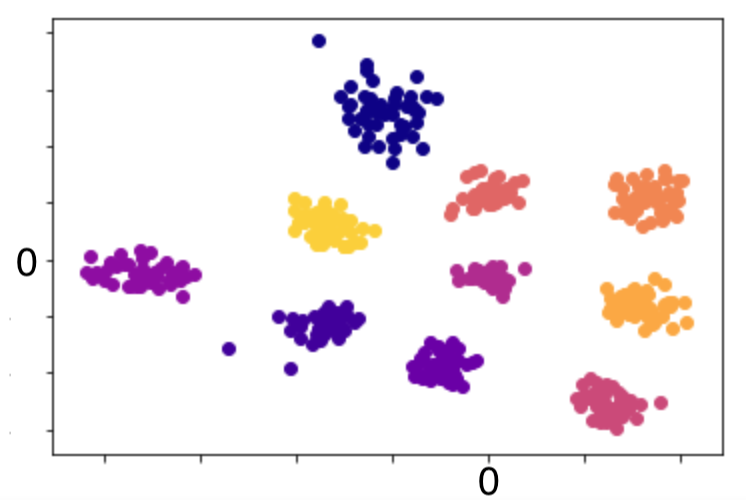} 
\includegraphics[width=.3\columnwidth]{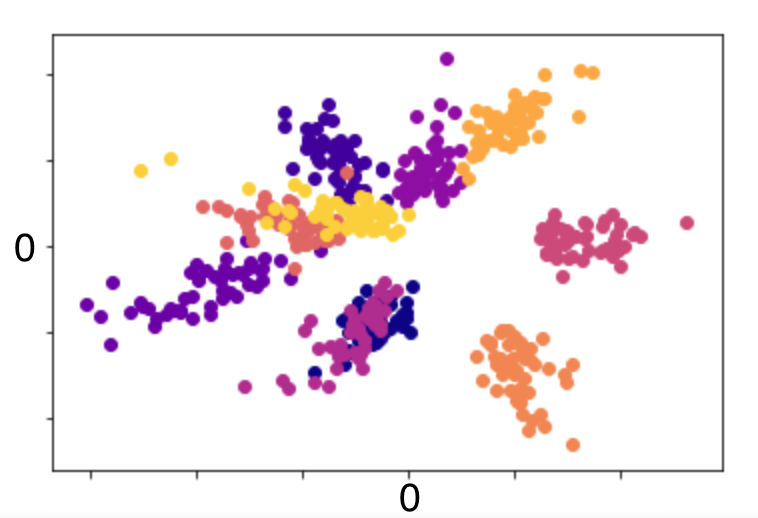} 
\includegraphics[width=.3\columnwidth]{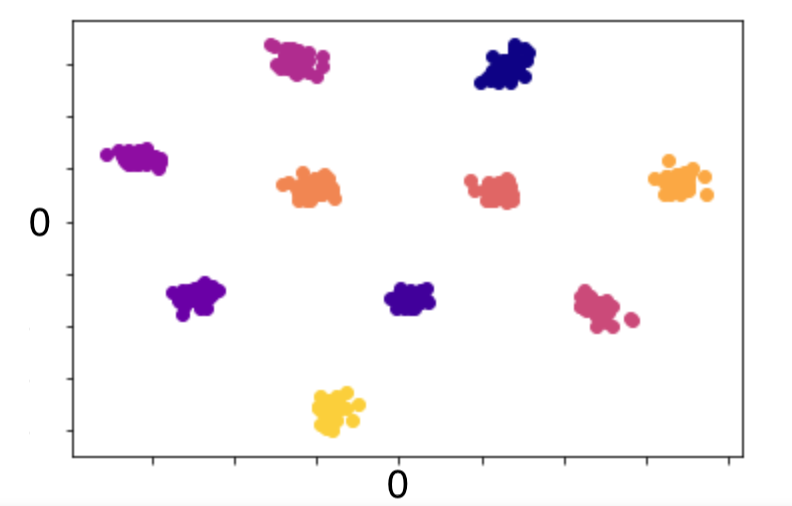} 
\caption{Visualizations of the feature space of a ResNet50 encoder (left) trained without OOS examples, (center) with OOS examples using a generic representation, (right) using a background class.}
\label{fig-feature-space}
\end{figure*}



In this section we analyze the geometry of the feature space induced by the use of a generic point to detect OOS examples. This is motivated by the empirical observation that the use of a generic point seems to change the way that a model clusters classes. In the low-dimensional setting of Figure \ref{fig-feature-space} for example, we observe that while moving from standard ProtoNets to ProtoNets with Background Detection appears to simply tighten clusters, moving to ProtoNets with a generic point results in a distinct radial cluster structure around the generic point (at $(0,0)$). Below we will try to formalize some of the intuition gained from these experiments so that we can make precise statements about the different kinds of geometry induced by these different problem formulations. Proofs for all propositions can be found in Section \ref{sect-proofs}. 

Recall that an {\emph{affine hyperplane in $\mathbb{R}^d$}} is a translation of a $(d-1)$-dimensional subspace. Alternatively, a non-zero vector $v \in \mathbb{R}^d$ and constant $b \in \mathbb{R}$ define an affine hyperplane via the expression $H := \{w \in \mathbb{R}^d \;|\;\langle w, v \rangle = b \}$. Note that any affine hyperplane $H$ decomposes $\mathbb{R}^d$ into two {\emph{open half-spaces}}: $H^+ := \{w \in \mathbb{R}^d \;|\;\langle w, v \rangle > b \}$ and $H^- := \{w \in \mathbb{R}^d \;|\;\langle w, v \rangle < b \}$. For any two distinct points $x_1,x_2 \in \mathbb{R}^d$, one gets a hyperplane $H_{x_1,x_2}$ defined by normal vector $x_1 - x_2$ and constant $\frac{1}{2}(||x_1||^2 - ||x_2||^2)$. In particular, when $\gamma_1$ and $\gamma_2$ are centroids for two classes, then $H_{\gamma_1,\gamma_2}$ is the decision boundary of the associated 2-way ProtoNets model (or alternatively the Voronoi partition corresponding to two points). 

Let $x$ be a point in $\mathbb{R}^d$ and let $\gamma_1, \dots, \gamma_{k}, \gamma_{oos}$ be a list of prototypes and generic point. Let $\mathcal{S}_{k+1}$ be the symmetric group on (or permutations of) $k+1$ elements. There is a trivial bijection between $\mathcal{S}_{k+1}$ and total orderings of $\gamma_1, \dots, \gamma_{k}, \gamma_{oos}$. In particular, for permutation $\sigma \in \mathcal{S}_{k+1}$, we associate $\sigma$ with the order $\gamma_{\sigma(1)} < \gamma_{\sigma(2)} < \dots < \gamma_{\sigma(oos)}$ where we write $\sigma(i) = j$ to represent the value $j \in \{1,\dots,oos\}$ that $\sigma$ permutes $i$ to (we use index $oos$ and $k+1$ interchangeably).

\begin{proposition}\label{prop-decomp}
Let $\gamma_1, \dots, \gamma_k, \gamma_{oos} \in \mathbb{R}^d$ be a finite list of prototypes and generic point. The set of hyperplanes corresponding to each pair of $\gamma_1, \dots, \gamma_k, \gamma_{oos}$ induce a decomposition of $\mathbb{R}^d$ into open (possibly empty) subsets ({\emph{cells}}) $S_{\sigma}$, where $\sigma \in \mathcal{S}_{k+1}$ and 
\begin{equation*}
    S_\sigma := \big\{x \in \mathbb{R}^d \;|\; ||x-\gamma_{\sigma(1)}|| < \dots < ||x-\gamma_{\sigma(oos)}||\big\},
\end{equation*}
as well as a measure zero, closed subset $B$ which is the union of all $H_{\gamma_i,\gamma_j}$ for $i,j \in \{1,\dots, k,oos\}$.
\end{proposition} 

The decomposition described in Proposition \ref{prop-decomp} can be used to describe those regions of $\mathbb{R}^d$ that can lead to the correct classification of an encoded point in different versions of the ProtoNet problem. As we will see, these regions differ substantially between the classic ProtoNets problem and ProtoNets with generic point. We call a point $x \in \mathbb{R}^d$, {\emph{$i$-viable}} if encoding a class $i$ query point $q$ such that $h_\theta(q) = x$ results in the correct prediction that $q$ belongs to class $i$, if class $i$ is represented in the support, or that $q$ is OOS, if class $i$ is not represented in the support. A point is called {\emph{viable}} if it is $i$-viable for some $i \in \{1,\dots,k\}$. A set of points $U$ is called {\emph{$i$-viable}} if every point in $U$ is $i$-viable and {\emph{viable}} if every point in $U$ is viable.

\begin{itemize}[leftmargin=13pt,topsep=0pt]
\item {\emph{Standard ProtoNets}}: For a point belonging to class $i$ to be predicted correctly, it must lie in a cell of the form $S_\sigma$ with $\sigma(1) = i$. Note that outside of measure-zero set $B$, every point in $\mathbb{R}^d$ is $i$-viable for some $i \in \{1,\dots,k\}$ since every cell $S_\sigma$ consists of points closest to some centroid (i.e. $\sigma(1) = j$ for some $j \in \{1,\dots,k\}$) and in the setting where OOS examples do not exist, a point belonging to class $i$ is always classified correctly if it is closer to centroid $\gamma_i$ than it is to any other centroid.
\item {\emph{ProtoNets with generic point}}: For a point belonging to class $i$ to be predicted correctly both when its prototype is present and also when it is not, it must satisfy inequality \eqref{eqn-generic-condition}.  This means that it must lie in a cell of the form $S_\sigma$ with $\sigma(1) = i$ and $\sigma(2) = oos$. Note that this condition means that even outside of $B$, there are non-viable regions of $\mathbb{R}^d$.
For example, if $\sigma(2) \neq oos$.
\end{itemize}

We illustrate these differences in Figure \ref{fig-decision-boundaries} for the standard ProtoNets problem (left) and  ProtoNets with generic point (right). We put boxes around the label of viable regions in each diagram.

\begin{figure*}[t]
\centering
\includegraphics[width=.77\columnwidth]{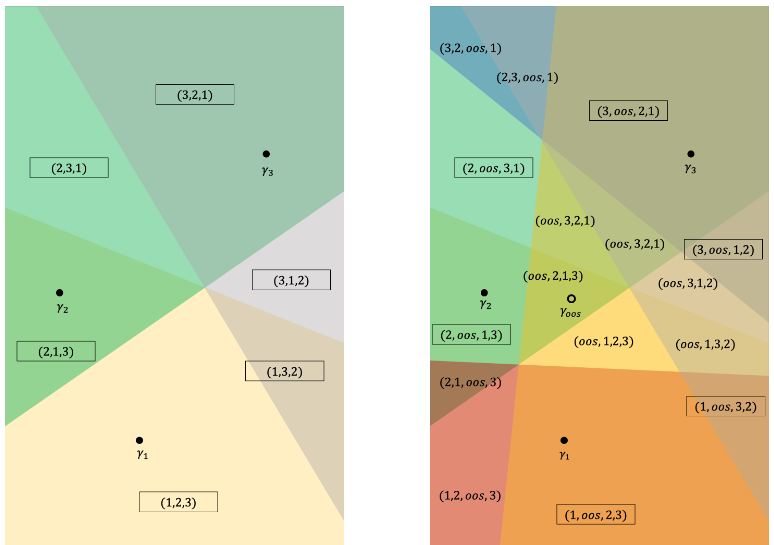}
\caption{Low dimensional illustrations of the feature space decision boundaries of the (left) standard ProtoNet problem with three prototypes, (right) the ProtoNet problem with generic point. In each region we label the ordering the closest prototypes/generic point. We box the labels of regions which are viable.}
\label{fig-decision-boundaries}
\end{figure*}

\begin{proposition} \label{prop-viable-regions}
Let $\{1,\dots,k\}$ be a set of classes and let $\gamma_{oos} \in \mathbb{R}^d$ be a generic point.
\begin{enumerate}[leftmargin=13pt,noitemsep,topsep=0pt]
    \item \label{item-prop-viable-regions} In the standard ProtoNets problem, the set of $i$-viable points is always nonempty for each choice of distinct prototypes $\gamma_1, \dots, \gamma_k \in \mathbb{R}^d$ and for all $i \in \{1,\dots,k\}$. 
    \item In the ProtoNets with generic point problem, there are choices $k$ and distinct $\gamma_1, \dots, \gamma_k, \gamma_{oos} \in \mathbb{R}^d$ for which the $i$-viable region of $\mathbb{R}^d$ is the empty set for some $i \in \{1,\dots,k\}$. There are also choices of distinct $\gamma_1, \dots, \gamma_k, \gamma_{oos}$ such that there is a nonempty $i$-viable region for each $i$.
\end{enumerate}
\end{proposition}

Thus we see that the introduction of a generic points puts additional constraints on how a model can arrange prototypes in feature space, with some arrangements being not only non-optimal, but actually precluding correct predictions.

Our final proposition shows that the radial pattern shown in Figure \ref{fig-feature-space} actually represents general geometric structure induced by the generic point problem. For fixed $\gamma_1, \dots, \gamma_k, \gamma_{oos} \in \mathbb{R}^d$ we call the region of $\mathbb{R}^d$ which consists of points that are closer to $\gamma_{oos}$ than to any $\gamma_1, \dots, \gamma_k$ the {\emph{OOS-core}} (note that as a corollary to Proposition \ref{prop-viable-regions}.\ref{item-prop-viable-regions} this always exists). We call two sets $U,V \subset \mathbb{R}^d$ {\emph{adjacent}} if there is a point $p \in \mathbb{R}^d$ such that for any $\epsilon > 0$, the open ball $B_\epsilon(p)$ contains points from both $U$ and $V$.

\begin{proposition} \label{prop-relationship between cells}
If $\gamma_1, \dots, \gamma_k, \gamma_{oos} \in \mathbb{R}^d$ are a choice of distinct prototypes/generic point such that the set $P_i$ of $i$-viable points is non-empty for $i \in \{1,\dots,k\}$, then $P_i$ is adjacent to the OOS core.
\end{proposition}

\section{Limitations and broader impacts}
\label{sect-limitation}

As indicated in Section \ref{subsect-generalization-experiments} all existing OOS detection methods struggle when applied to data that is significantly different from the model's test set. This is a major obstacle blocking the use of these models in many science domains. While Centered GROOS represents progress, we suspect that more sophisticated methods will need to be developed in the future in order to achieve satisfactory results. Additionally, our theoretical analysis only addressed limited aspects of the GROOS model. In Section \ref{subsect-how-do-prototypes-work} for example, we restrict ourselves to a study of fixed prototypes. Finally, in future work we would like to compare our metric-based approaches to OOS detection, to metalearning methods developed to address related problems \cite{jeong2020ood}. 

Our work, like other work in few-shot learning, has the capability to broaden access to deep learning tools by lowering training data requirements. This can be positive because it means that groups without large data acquisition and computing budgets can apply deep learning to solve problems and negative because those uses of machine learning with negative societal impacts now require less data.


\section{Conclusion}

In many situations, the ability to detect OOS examples is a necessary requirement for deployment of few-shot learning models. In this paper we showed that in the metric-based setting, GROOS and its variant Centered GROOS are two methods that begin to address this challenge. Despite the fact that our models, on average, outperformed existing approaches, we believe OOS detection is a challenge that deserves more attention within the few-shot community, since effective solutions will enable broader adoption of few-shot methods for real-world science and engineering applications.


\section*{Acknowlegments}
This work was funded by the U.S. Government.

\bibliographystyle{plain}
\bibliography{OOD}

\begin{thebibliography}{10}

\bibitem{chen2019closer}
Wei-Yu Chen, Yen-Cheng Liu, Zsolt Kira, Yu-Chiang~Frank Wang, and Jia-Bin
  Huang.
\newblock A closer look at few-shot classification.
\newblock {\em arXiv:1904.04232}, 2019.

\bibitem{chowdhury2021few}
Arkabandhu Chowdhury, Mingchao Jiang, and Chris Jermaine.
\newblock Few-shot image classification: Just use a library of pre-trained
  feature extractors and a simple classifier.
\newblock {\em arXiv preprint arXiv:2101.00562}, 2021.

\bibitem{cimpoi14describing}
M.~Cimpoi, S.~Maji, I.~Kokkinos, S.~Mohamed, , and A.~Vedaldi.
\newblock Describing textures in the wild.
\newblock In {\em Proceedings of the {IEEE} Conf. on Computer Vision and
  Pattern Recognition ({CVPR})}, 2014.

\bibitem{finn2017model}
Chelsea Finn, Pieter Abbeel, and Sergey Levine.
\newblock Model-agnostic meta-learning for fast adaptation of deep networks.
\newblock In {\em Proceedings of the 34th International Conference on Machine
  Learning-Volume 70}, pages 1126--1135. JMLR. org, 2017.

\bibitem{hariharan2017low}
Bharath Hariharan and Ross Girshick.
\newblock Low-shot visual recognition by shrinking and hallucinating features.
\newblock In {\em Proceedings of the IEEE International Conference on Computer
  Vision}, pages 3018--3027, 2017.

\bibitem{hartshorne2013algebraic}
Robin Hartshorne.
\newblock {\em Algebraic geometry}, volume~52.
\newblock Springer Science \& Business Media, 2013.

\bibitem{resnet2016}
Kaiming He, X.~Zhang, Shaoqing Ren, and Jian Sun.
\newblock Deep residual learning for image recognition.
\newblock {\em 2016 IEEE Conference on Computer Vision and Pattern Recognition
  (CVPR)}, pages 770--778, 2016.

\bibitem{hendrycks2020many}
Dan Hendrycks, Steven Basart, Norman Mu, Saurav Kadavath, Frank Wang, Evan
  Dorundo, Rahul Desai, Tyler Zhu, Samyak Parajuli, Mike Guo, et~al.
\newblock The many faces of robustness: A critical analysis of
  out-of-distribution generalization.
\newblock {\em arXiv preprint arXiv:2006.16241}, 2020.

\bibitem{hendrycks2016baseline}
Dan Hendrycks and Kevin Gimpel.
\newblock A baseline for detecting misclassified and out-of-distribution
  examples in neural networks.
\newblock {\em International Conference on Learning Representations}, 2017.

\bibitem{jeong2020ood}
Taewon Jeong and Heeyoung Kim.
\newblock {OOD}-{MAML}: Meta-learning for few-shot out-of-distribution
  detection and classification.
\newblock {\em Advances in Neural Information Processing Systems}, 33, 2020.

\bibitem{kornblith2019better}
Simon Kornblith, Jonathon Shlens, and Quoc~V Le.
\newblock Do better imagenet models transfer better?
\newblock In {\em Proceedings of the IEEE/CVF Conference on Computer Vision and
  Pattern Recognition}, pages 2661--2671, 2019.

\bibitem{krizhevsky2009learning}
Alex Krizhevsky, Vinod Nair, and Geoffrey Hinton.
\newblock Learning multiple layers of features from tiny images.
\newblock Technical report, University of Toronto, 2009.

\bibitem{Lake1332}
Brenden~M. Lake, Ruslan Salakhutdinov, and Joshua~B. Tenenbaum.
\newblock Human-level concept learning through probabilistic program induction.
\newblock {\em Science}, 350(6266):1332--1338, 2015.

\bibitem{lakshminarayanan2017simple}
Balaji Lakshminarayanan, Alexander Pritzel, and Charles Blundell.
\newblock Simple and scalable predictive uncertainty estimation using deep
  ensembles.
\newblock In {\em Proceedings of the 31st International Conference on Neural
  Information Processing Systems}, pages 6405--6416, 2017.

\bibitem{LeCun2015}
Yann LeCun, Yoshua Bengio, and Geoffrey Hinton.
\newblock Deep learning.
\newblock {\em Nature}, page 436–444, 2015.

\bibitem{lecun1998gradient}
Yann LeCun, L{\'e}on Bottou, Yoshua Bengio, and Patrick Haffner.
\newblock Gradient-based learning applied to document recognition.
\newblock {\em Proceedings of the IEEE}, 86(11):2278--2324, 1998.

\bibitem{lee2018training}
Kimin Lee, Honglak Lee, Kibok Lee, and Jinwoo Shin.
\newblock Training confidence-calibrated classifiers for detecting
  out-of-distribution samples.
\newblock In {\em International Conference on Learning Representations}, 2018.

\bibitem{maji13fine-grained}
S.~Maji, J.~Kannala, E.~Rahtu, M.~Blaschko, and A.~Vedaldi.
\newblock Fine-grained visual classification of aircraft.
\newblock Technical report, Johns Hopkins University, 2013.

\bibitem{marcel2010torchvision}
S{\'e}bastien Marcel and Yann Rodriguez.
\newblock Torchvision the machine-vision package of torch.
\newblock In {\em Proceedings of the 18th ACM international conference on
  Multimedia}, pages 1485--1488, 2010.

\bibitem{murecsan2018fruit}
Horea Mure{\c{s}}an and Mihai Oltean.
\newblock Fruit recognition from images using deep learning.
\newblock {\em Acta Universitatis Sapientiae, Informatica}, 10(1):26--42, 2018.

\bibitem{nichol2018first}
Alex Nichol, Joshua Achiam, and John Schulman.
\newblock On first-order meta-learning algorithms.
\newblock {\em arXiv:1803.02999}, 2018.

\bibitem{snell2017prototypical}
Jake Snell, Kevin Swersky, and Richard Zemel.
\newblock Prototypical networks for few-shot learning.
\newblock In {\em Advances in neural information processing systems}, pages
  4077--4087, 2017.

\bibitem{sung2018learning}
Flood Sung, Yongxin Yang, Li~Zhang, Tao Xiang, Philip~HS Torr, and Timothy~M
  Hospedales.
\newblock Learning to compare: Relation network for few-shot learning.
\newblock In {\em Proceedings of the IEEE Conference on Computer Vision and
  Pattern Recognition}, pages 1199--1208, 2018.

\bibitem{NIPS2016_6385}
Oriol Vinyals, Charles Blundell, Timothy Lillicrap, koray kavukcuoglu, and Daan
  Wierstra.
\newblock Matching networks for one shot learning.
\newblock In D.~D. Lee, M.~Sugiyama, U.~V. Luxburg, I.~Guyon, and R.~Garnett,
  editors, {\em Advances in Neural Information Processing Systems 29}, pages
  3630--3638. Curran Associates, Inc., 2016.

\bibitem{WahCUB_200_2011}
C.~Wah, S.~Branson, P.~Welinder, P.~Perona, and S.~Belongie.
\newblock {The Caltech-UCSD Birds-200-2011 Dataset}.
\newblock Technical Report CNS-TR-2011-001, California Institute of Technology,
  2011.

\bibitem{wang2019out}
Kuan-Chieh Wang, Paul Vicol, Eleni Triantafillou, Chia-Cheng Liu, and Richard
  Zemel.
\newblock Out-of-distribution detection in few-shot classification.
\newblock {\em OpenReview.net}, 2019.

\bibitem{Wang2018}
Yaqing Wang, Quanming Yao, James Kwok, and Lionel~M. Ni.
\newblock {Generalizing from a Few Examples: A Survey on Few-Shot Learning}.
\newblock In {\em Intelligent Systems Design and Applications}, pages 100--112.
  Springer, 2018.

\bibitem{zhang2017universum}
Xiang Zhang and Yann LeCun.
\newblock Universum prescription: Regularization using unlabeled data.
\newblock In {\em Proceedings of the AAAI Conference on Artificial
  Intelligence}, volume~31, 2017.

\end{thebibliography}

\newpage

\appendix

\section{Appendix}

\subsection{Encoder size} \label{appendix-encoder-size}

Given that much of the metric-based few-shot learning literature uses small encoders, in Figure \ref{fig-encoder-size-same-train-test} we include results for the ``standard'' few-shot experiments using a 4-Conv encoder (as in  \cite{snell2017prototypical,wang2019out}) rather than the ResNet18 encoder used in Section \ref{sect-experiments-standard}. Interestingly, we find that with a smaller encoder the LCBO method does significantly better relative to other approaches, indicating that learning decision boundaries for OOS detection may be a more effective strategy in either a lower dimensional feature space or for less rich encoders. In future work it would be interesting to investigate whether attaching a larger MLP helps LCBO scale to larger encoders. Of course, including more fully-connected layers quickly becomes expensive which would be a potential downside of this method.


\begin{table*}
    \centering
    {\def\arraystretch{1.5}%
    \setlength{\tabcolsep}{4pt}
    \begin{tabular}{ccccccc}
 & \multicolumn{2}{c}{CIFAR100}  & \multicolumn{2}{c}{CUB-200} & \multicolumn{2}{c}{Omniglot} \\
 & AUPR & AUROC  & AUPR & AUROC  & AUPR & AUROC \\
 \hline
MinDist & 66.85 & 67.18 & 64.51 & 66.73 & 95.94 & 95.90  \\
LCBO & \textbf{75.15$^*$} & \textbf{75.89$^*$} & \textbf{68.98$^*$}  & \textbf{71.96$^*$} & \textbf{98.98$^*$} & \textbf{99.10$^*$}    \\
\hline
\textbf{Background (ours)} & 67.71 & 65.97 & 61.26 & 60.38 & 98.38 & 98.29\\
\textbf{GROOS (ours)} & 74.15 & 74.95 & 67.27 & 66.55 & \textbf{98.81} & \textbf{98.77}  \\
\textbf{Centered GROOS (ours)} & 70.36 & 71.41 & 65.34  & 65.38 & \textbf{98.65} & \textbf{98.62} \\
\hline
\end{tabular}}
    \caption{Results for the same set of experiments reported in Table \ref{fig-same-dataset} but using a 4-Conv encoder rather than a ResNet18 encoder. We put an $*$ next to the top score in each column and set in bold all the rest of the scores in the column that are within $0.5$ of this.}
    \label{fig-encoder-size-same-train-test}
\end{table*}

We also repeated the generalization experiments from Section \ref{subsect-generalization-experiments} (which also used a ResNet18 encoder) with a $4$-Conv encoder and a ResNet50 encoder. We summarize our results in Figures \ref{fig-generalization-other-encoders-AUROC} and \ref{fig-generalization-other-encoders-AUPR}. The logic behind our choice to also test larger encoders in this setting stemmed from the observation that in tasks that require higher levels of generalization, large encoders can sometimes yield better results \cite{hendrycks2020many}. We find that larger encoders do tend to slightly improve performance in terms of AUROC and AUPR. With the exception of AUPR for the Aircraft dataset where MinDist performed slightly better than Centered GROOS when we used a ResNet50 encoder instead of a ResNet18 encoder, the top performing model on a dataset did not change based on whether one used a larger encoder. It is perhaps notable that the Aircraft dataset is also one of the few examples where model performance decreased when using a ResNet50 encoder rather than a ResNet18 encoder.

Distinct from the pattern we observed in Figure \ref{fig-encoder-size-same-train-test}, in this setting using a smaller encoder did not appear to result in much better performance for LCBO. With the exception of its performance on ImageNet itself, which does not require the same level of generalization, LCBO did not out-perform other methods on any of the datasets. We suspect that this arises from the fact that learning decision boundaries is not an approach that transfers well to significantly different datasets. Similar to the results from Section \ref{subsect-generalization-experiments} we observe that the top models in terms of generalization were Centered GROOS and MinDist suggesting that centered generic points and raw distance are better able to capture ``different-ness'' across datasets. We also observe that in the smaller encoder setting, MinDist is more competitive with Centered GROOS.

\begin{table*}
    \centering
    {\def\arraystretch{1.5}%
    \setlength{\tabcolsep}{4pt}
    \begin{tabular}{cccccccc}
&  ImageNet & CIFAR100  & Omniglot &  Aircraft & Textures & Fruits\\
 \hline
  & & 4-Conv encoder & & & & \\
MinDist & 71.28 & 63.36 & \textbf{67.45$^*$} & 54.07 & 57.29 & \textbf{95.97$^*$} \\
LCBO &  \textbf{75.51} & 60.36 & 58.68 & 52.49 & 58.16 & 87.35 \\
\textbf{Background (ours)} & 61.78 & 60.78 & 65.10 & 52.95 & 55.75 & 92.57\\
\textbf{GROOS (ours)} & \textbf{75.99$^*$} & 59.22 & 61.44 & 53.59 & 59.44 & 90.88 \\
\textbf{Centered GROOS (ours)} & 61.78 & \textbf{64.80$^*$} & \textbf{67.03} & \textbf{54.87$^*$} & \textbf{61.77$^*$} & 94.80 \\
\hline
 & & ResNet50 encoder & & & & \\
MinDist & \textbf{97.51} & 82.33 & \textbf{84.54$^*$} & 58.49 & 76.46 & 92.57 \\
LCBO & 95.66 & 79.54 & 73.93 & 55.72 & 74.48 & 90.72 \\
\textbf{Background (ours)} & 95.57 & 79.11 & 60.61 & 53.86 & 78.48 & 92.30 \\
\textbf{GROOS (ours)} & \textbf{97.76$^*$} & 79.17 & 78.55 & 53.38 & 80.60 & 94.19 \\
\textbf{Centered GROOS (ours)} & 96.96 & \textbf{84.20$^*$} & 81.36 & \textbf{59.07$^*$} & \textbf{84.17$^*$} & \textbf{96.40$^*$} \\
\end{tabular}}
    \caption{AUROC results for the same set of experiments reported on in Table \ref{fig-generalization-AUROC} but using a 4-Conv encoder (top) and ResNet50 encoder (bottom) rather than a ResNet18 encoder. We put an $*$ next to the top score in each column and set in bold all the rest of the scores in the column that are within $0.5$ of this.}
    \label{fig-generalization-other-encoders-AUROC}
\end{table*}

\begin{table*}
    \centering
    {\def\arraystretch{1.5}%
    \setlength{\tabcolsep}{4pt}
    \begin{tabular}{cccccccc}
&  ImageNet & CIFAR100  & Omniglot &  Aircraft & Textures & Fruits\\
 \hline
  & & 4-Conv encoder & & & & \\
MinDist & 70.44 & 62.67 & \textbf{69.71$^*$} & \textbf{53.84$^*$} & 57.27 & \textbf{95.97$^*$}  \\
LCBO & \textbf{74.50$^*$} & 58.58 & 57.04  & 52.20 & 57.31 &  87.48  \\
\textbf{Background (ours)} & 60.67  & 59.07 & 61.99 & 52.39 & 55.75 & 91.66 \\
\textbf{GROOS (ours)} & \textbf{75.48} & 58.62 & 60.55 & 53.10 & 59.12 & 91.66  \\
\textbf{Centered GROOS (ours)} & 70.44 & \textbf{63.17$^*$} & 64.04 & \textbf{53.73} & \textbf{59.46$^*$} & 94.73 \\
\hline
 & & ResNet50 encoder & & & & \\
MinDist & \textbf{97.63} & 81.41 & \textbf{85.62$^*$} &  \textbf{58.72$^*$} & 78.77 & 91.66  \\
LCBO & 95.28 &77.56  & 71.53 & 55.17 & 73.48 & 91.25   \\
\textbf{Background (ours)} & 95.63  & 79.26 & 77.26 & 57.87 & 77.09 & 95.08 \\
\textbf{GROOS (ours)} & \textbf{97.68$^*$} & 77.16 & 79.53 & 54.63 &76.87  & 94.51  \\
\textbf{Centered GROOS (ours)} & 96.75 & \textbf{83.27$^*$} & 81.19 & \textbf{58.24} & \textbf{82.02$^*$} & \textbf{96.48$^*$} \\
\end{tabular}}
    \caption{AUPR results for the same set of experiments reported on in Table \ref{fig-generalization-AUPR} but using a 4-Conv encoder (top) and ResNet50 encoder (bottom) rather than a ResNet18 encoder. We put an $*$ next to the top score in each column and set in bold all the rest of the scores in the column that are within $0.5$ of this.}
    \label{fig-generalization-other-encoders-AUPR}
\end{table*}

\subsection{Proofs from Section \ref{subsect-how-do-prototypes-work}}
\label{sect-proofs}

\begin{proof}[Proof of Proposition \ref{prop-decomp}]
For any $x \in \mathbb{R}^d$, either (1) there are at least two $\gamma_i, \gamma_j$ for $i,j \in \{1,\dots,k,oos\}$ such that $||x - \gamma_i|| = ||x - \gamma_j||$ or (2) for all $\gamma_i, \gamma_j$ either $||x-\gamma_i|| > ||x - \gamma_j||$ or $||x-\gamma_i||<||x-\gamma_j||$. In the former case, $x \in B$ since $x$ belongs to $H_{\gamma_i,\gamma_j}$ as this hyperplane consists precisely of those $x'$ such that $||x'-\gamma_i|| = ||x' - \gamma_j||$. In the latter case the set  
\begin{equation*}
D = \big\{||x - \gamma_i|| \;| \; i \in \{1,\dots,k,oos\}\big\}
\end{equation*}
consists of distinct real numbers. It is clear that these numbers can be ordered so that they are strictly increasing. Denote by $\sigma$ the permutation from $\mathcal{S}_{k+1}$ such that
\begin{equation*}
||x-\gamma_{\sigma(1)}|| < ||x-\gamma_{\sigma(2)}|| < \dots < ||x-\gamma_{\sigma(oos)}||.
\end{equation*}
Then $x \in S_\sigma$. This shows that the union of $B$ and each set in $\{S_\sigma \;|\; \sigma \in \mathcal{S}_{k+1}\}$ is equal to $\mathbb{R}^d$. Using the distance parametrization of each $S_\sigma$ based on $\sigma$, it is also clear that $B$ is disjoint from each $S_\sigma$ and that furthermore, $S_\sigma \cap S_\tau = \emptyset$ when $\sigma \neq \tau$.

The fact that each $S_\sigma$ is open, and $B$ is closed and measure zero follows from elementary topology/measure theory.

\end{proof}

\begin{proof}[Proof of Proposition \ref{prop-viable-regions}] \mbox{}
\begin{enumerate} 
\item If $\gamma_1, \dots, \gamma_k$ are distinct from each other, then for any $i \in \{1,\dots,k\}$, we can choose $\epsilon > 0$ sufficiently small such that for all points $x \in B_\epsilon(\gamma_i)$ we have that $||x - \gamma_i|| < ||x - \gamma_j||$ for each $j \in \{1,\dots,k\}$ with $j \neq i$. Observe that 
\begin{equation*}
B_\epsilon(x) \subseteq \bigcup_{\substack{\sigma \in \mathcal{S}_{k+1} \\ \sigma(1) = i}} S_\sigma,
\end{equation*}
that is, $B_\epsilon(x)$ belongs to the $i$-viable region of $\mathbb{R}^d$. Hence the $i$-viable region is non-empty.
\item We give two examples, in the first there exists an element $i \in \{1,\dots,k\}$ such that the $i$-viable region is empty. In the second, for each $i \in \{1,\dots,k\}$, the $i$-viable region is not empty. In both cases we leave it to the reader to verify the example.
\begin{enumerate}
\item Consider the case $d = 2$, $k=2$, $\gamma_{oos} = (1,0), \gamma_1 = (0,0)$, and $\gamma_2 = (-1,0)$. It can be checked that in this case the $2$-viable region consists of those points that are both to the left of the line $x = (0,0)$ and to the right of the line $x = (\frac{1}{2},0)$. This set is of course empty. 
\item Consider the case $d = 2$, $k=4$, $\gamma_{oos} = (0,0), \gamma_1 = (1,0), \gamma_2 = (0,1), \gamma_3 = (-1,0)$, and $\gamma_4 = (0,-1)$. Elementary calculations show that the $1$-viable region is nonzero and defined by the inequalities $y > -\frac{1}{2}$, $y < \frac{1}{2}$, and $x > \frac{1}{2}$. The $2$, $3$, and $4$-viable regions can be obtained from the $1$-viable region via symmetry transformations.
\end{enumerate}
\end{enumerate}
\end{proof}



To prove Proposition \ref{prop-relationship between cells}, we need to establish a couple short lemmas:

\begin{lemma}\label{lem:hyperplane}
    Let $\gamma_i$ and $
    \gamma_j$ be distinct prototypes. If two points $x$ and $y$ satisfy the inequalities
\begin{equation*}
||x-\gamma_i||<||x-\gamma_j||\quad\text{and}\quad||y-\gamma_i||<||y-\gamma_j||,
\end{equation*}
    then for any point $z$ on the line segment $\ell$ connecting these two points,
\begin{equation}
\label{lemma-inequality}
    ||z-\gamma_i||<||z-\gamma_j||.
\end{equation}
    If, instead,
\begin{equation*}
    ||x-\gamma_i||=||x-\gamma_j||\quad\text{and}\quad||y-\gamma_i||<||y-\gamma_j||,
\end{equation*}
the strict inequality \eqref{lemma-inequality} holds at every point on $\ell \setminus \{x\}.$
\end{lemma}
\begin{proof}
    To prove the first part of the Lemma, notice that both $x$ and $y$ lie on the same side of the hyperplane $H_{\gamma_i,\gamma_j}$. Since a hyperplane splits $\mathbb{R}^d$ into two convex half-spaces, the entire segment $\ell$ lies on a single side of this hyperplane and the result follows. 
    
    The last statement is true since, if the segment does not lie entirely in the plane $H_{\gamma_i,\gamma_j}$, it can only intersect at a single point, $x$ (note that $\ell$ could also lie entirely within $H_{\gamma_i,\gamma_j}$ but we know that the other end point of $\ell$, $y$, is not in $H_{\gamma_i,\gamma_j}$). Since $x$ is the endpoint of the segment, the rest lies in one of the open half spaces, in this case, that whose points satisfy \eqref{lemma-inequality}.
\end{proof}

    

\begin{lemma}\label{lem:midpoint}
    Let the $\gamma_i, \gamma_j$ be as in Lemma~\ref{lem:hyperplane}, $x$ a point in the $i$-viable region and $\gamma_{oos}$ be a distinct generic representation. Let $\ell$ be the line segment between $x$ and $\gamma_{oos}$ and $z$ be the point on $\ell$ where it intersects  $H_{\gamma_{oos}\gamma_i}$. Then the line segment $\ell'$ from $x$ to $z$ is such that for any point $w$ on this segment and for all $j \in \{1,\dots,k\}$ with $j \neq i$,
    \begin{equation*}
        ||w-\gamma_i||<||w-\gamma_{oos}||<||w-\gamma_j||.
    \end{equation*}
    Similarly, if $\ell''$ is the line segment from $z$ to $\gamma_{oos}$, then all $w$ on $\ell''$ satisfy 
    \begin{equation*}
        ||w - \gamma_{oos}|| < ||w - \gamma_j||
    \end{equation*}
    for all $j \in \{1,\dots,k\}$ (including $j = i$).
\end{lemma}
\begin{proof}
    Notice that $\gamma_{oos}$ and $x$ satisfy the inequalities
    \begin{equation*}
    0 = ||\gamma_{oos}-\gamma_{oos}||<||\gamma_{oos}-\gamma_j||\quad\text{and}\quad ||x-\gamma_{oos}||<||x-\gamma_j||
    \end{equation*}
    for any $j \in \{1,\dots,k\}$ with $j\ne i$. Applying Lemma~\ref{lem:hyperplane}, this implies that $z$, which lies on the line segment connecting $x$ an $\gamma_{oos}$, satisfies 
    \begin{equation*}
        ||z - \gamma_{oos}|| < ||z - \gamma_j||.
    \end{equation*}
    Since it lies on $H_{\gamma_{oos},\gamma_i}$ as well, 
    \begin{equation}\label{eq-midpoint}
        ||z-\gamma_i||=||z-\gamma_{oos}||<||z-\gamma_j||.
    \end{equation}
    
    But since $x$ satisfies
    \begin{equation*}
        ||x-\gamma_i||<||x-\gamma_{oos}||<||x-\gamma_j||,
    \end{equation*}
    two applications of Lemma~\ref{lem:hyperplane} yield that for any point $w$ on $\ell'$,
    \begin{equation*}
        ||w-\gamma_i||<||w-\gamma_{oos}||<||w-\gamma_j||.
    \end{equation*}
    This proves the first statement.
    
    Next, returning to \eqref{eq-midpoint}, we see that since
    \begin{equation*}
        0 = ||\gamma_{oos} - \gamma_{oos}|| < ||\gamma_{oos} - \gamma_{j}||
    \end{equation*}
    for any $j \in \{1,\dots,k\}$ (including $i = j$), then by Lemma~\ref{lem:hyperplane}, for all $w$ on $\ell''$ we must have that
    \begin{equation*}
        ||w - \gamma_{oos}|| < ||w -\gamma_{j}||,
    \end{equation*}
    which proves the second statement.
\end{proof}

Using these lemmas, we can prove Proposition~\ref{prop-relationship between cells}:
\begin{proof}
    Let $x$ be a point in the $i$-viable region of $\mathbb{R}^d$ and $z$ be the point on the segment $\ell$ between $x$ and $\gamma_{oos}$ that lies on the hyperplane $H_{\gamma_i,\gamma_{oos}}$. Note that $\ell$ must cross this hyperplane since $x$ lies on one side of $H_{\gamma_i,\gamma_{oos}}$, being closer to $\gamma_i$ than to $\gamma_{oos}$, and $\gamma_{oos}$ lies on the other. 
    
    By Lemma~\ref{lem:midpoint}, all points $w$ of $\ell$ on the same side of $H_{\gamma_i,\gamma_{oos}}$ as $x$ satisfy 
    \begin{equation*}
        ||w-\gamma_i||<||w-\gamma_{oos}||<||w-\gamma_j||,
    \end{equation*}
    for all $j \in \{1,\dots,k\}$ with $j \neq i$. All such points are $i$-viable. All points on $\ell$ on the same side of $H_{\gamma_i,\gamma_{oos}}$ as $\gamma_{oos}$ satisfy 
    \begin{equation*}
        ||w-\gamma_{oos}|| < ||w-\gamma_{j}||
    \end{equation*}
    for all $j \in \{1,\dots,k\}$ including $j = i$. It follows that these points are in the OOS-core. It is clear then that for any $\epsilon > 0$, the ball $B_{\epsilon}(z)$ contains both points from the $i$-viable region of $\mathbb{R}^d$ and the OOS-core. This proves the Proposition.
    
    


\end{proof}

    
    
\end{document}